\documentclass{article}

\usepackage{graphicx}
\usepackage{subfigure}
\usepackage{booktabs} 
\usepackage{amssymb}
\usepackage{amsmath}
\usepackage{amsthm}
\usepackage{algorithm}
\usepackage{algorithmic}
\usepackage{wrapfig}
\usepackage{multibib}
\usepackage{array, multirow, boldline}

\usepackage[preprint, nonatbib]{neurips_2020}

\usepackage[utf8]{inputenc} 
\usepackage[T1]{fontenc}    
\usepackage{hyperref}       
\usepackage{url}            
\usepackage{booktabs}       
\usepackage{amsfonts}       
\usepackage{nicefrac}       
\usepackage{microtype}      

\newtheorem{theom}{Theorem}

\newtheorem{assumption}{Assumption}
\newtheorem{lemma}{Lemma}

\usepackage{color}
\definecolor{comment}{rgb}{1.0,0,0}
\definecolor{change}{rgb}{0,0,1}
\definecolor{normal}{rgb}{0, 0, 0}

\newcites{App}{Appendix References}

\title{Attentive Gaussian processes for probabilistic time-series generation}

\author{%
  Kuilin Chen \\
  University of Toronto\\
  Toronto, Ontario, Canada \\
  \texttt{kuilin.chen@mail.utoronto.ca} \\
  \And
  Chi-Guhn Lee \\
  University of Toronto\\
  Toronto, Ontario, Canada \\
  \texttt{cglee@mie.utoronto.ca} \\
}

\begin{document}

\maketitle

\begin{abstract}
  The transduction of sequence has been mostly done by recurrent networks, which are computationally demanding and often underestimate uncertainty severely. We propose a computationally efficient attention-based network combined with the Gaussian process regression to generate real-valued sequence, which we call the Attentive-GP. The proposed model not only improves the training efficiency by dispensing recurrence and convolutions but also learns the factorized generative distribution with Bayesian representation. However, the presence of the GP precludes the commonly used mini-batch approach to the training of the attention network. Therefore, we develop a block-wise training algorithm to allow mini-batch training of the network while the GP is trained using full-batch, resulting in a scalable training method. The algorithm has been proved to converge and shows comparable, if not better, quality of the found solution. As the algorithm does not assume any specific network architecture, it can be used with a wide range of hybrid models such as neural networks with kernel machine layers in the scarcity of resources for computation and memory.
\end{abstract}

\section{Introduction} \label{sec:introduction}
The deep generative models are originally designed for unconditional generation tasks using deterministic neural networks and stochastic latent variables. Given training samples $\mathbf{y}$ in a potentially high-dimensional space, a generative model learns to represent an estimation of the true generative distribution $p_{data}(\mathbf{y})$ and generates new samples from the learnt representation. Direct modelling $p_{data}$ is intractable due to high-dimensionality and complex correlation across dimensions. Instead of learning the complex distribution directly, latent variable models have been developed to learn a deterministic mapping from simple stochastic latent variables $\mathbf{z}$ to observed variables $\mathbf{y}$. Both generative adversarial networks (GANs) \cite{goodfellow:2014} and variational autoencoders (VAEs) \cite{Kingma:Welling:2013} fall into the category of latent variable models, and can be extended to conditional generation tasks \cite{mirza:2014,sohn:2015}, where $\mathbf{y}$ is dependent on external input $\mathbf{x}$. A fundamental issue in those latent variable-based conditional generative models is that the latent variables could be ignored \cite{sonderby:etal:2016}. As a universal approximator, the generative network is capable of learning $\mathbf{y}$ directly from $\mathbf{x}$ and treats $\mathbf{z}$ as noise. When such bypassing happens for latent variables, the generative models degenerate to deterministic models. Although the bypassing phenomenon can be alleviated by introducing weakened generative networks \cite{bowman:etal:2016,chen:etal:2017lossy,serban:etal:2017}, it requires manual and problem-specific design in generative network architectures.

The auto-regressive generative approach models the output sequence $\{{y}_1,...,{y}_{L}\}$ conditional on the input sequence $\{\mathbf{x}_1,...,\mathbf{x}_L\}$ and avoids the bypassing phenomenon \cite{van:eta;:2016pixelcnn}. The generation task is performed by iteratively computing a new output $\hat{y}_t$ for $t=1,...,L$ given previously computed outputs $\{\hat{y}_1,...,\hat{y}_{t-1}\}$ and the entire input sequence $\{\mathbf{x}_1,...,\mathbf{x}_L\}$ assuming the generative distribution $p_{data}(\mathbf{y})$ can be factorized as $p(y_1)p(y_2|y_1)...p(y_L|y_1,...,y_{L-1})$ across $L$ time steps \cite{van:etal:2016pixelrnn}.

An issue in conditional generative models parameterized by neural networks is that the variance for out-of-sample inputs can be arbitrary (see Fig. \ref{fig:sin_example}(a)), where neural networks assign unreasonably small uncertainty over incorrectly predicted mean values for out-of-sample data points. Therefore, it is undesirable to model the generative distribution for real-valued time-series via standard neural networks.

\begin{wrapfigure}{r}{0.5\textwidth}
    \centering
      \centering
      \includegraphics[width=1.0\linewidth]{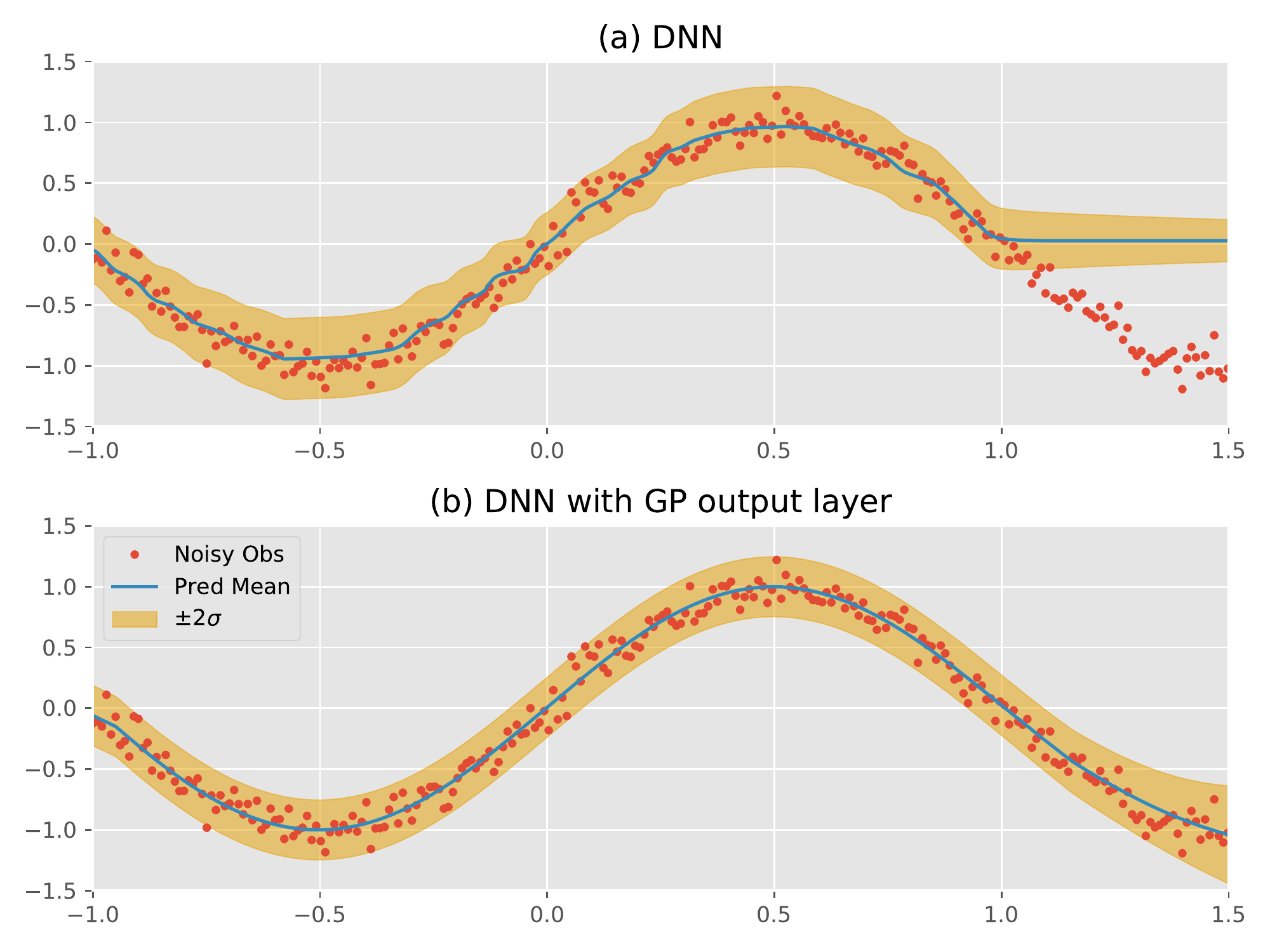}
      \caption{Predictive distributions for $sin(x)$ by (a) deep neural networks (b) deep neural networks with GP output layer. Both models are trained by data between $[-1, 1]$. The data between $(1, 1.5]$ is unseen during training.}
      \label{fig:sin_example}
\end{wrapfigure}    
    
Nonparametric Bayesian approaches, such as Gaussian processes (GPs), provide richer representation of uncertainty through kernel functions \cite{Rasmussen:Williams:2006}. Data outside the boundary of the training set can be well extrapolated by GPs with appropriate kernels \cite{wilson:ryan:2013}. Given the intuitive benefit of combining GPs and neural networks, such hybrid models have been considered in different contexts. For example, GP regression can be integrated into feedforward and recurrent neural networks as the last layer for probabilistic regression tasks \cite{calandra:etal:2016,wilson:etal:2016}. Although hybrid models outperform standalone neural networks for i.i.d. prediction tasks, the potential application in time-series generation has not been investigated. We find that deep neural networks (DNN) with GP output layers lead to robust estimation of data distribution for both in-sample and out-of-sample inputs (see Fig. \ref{fig:sin_example} (b)). Such property is desirable when generating time-series given out-of-sample inputs.

In this paper, we tackle the existing challenges in conditional time-series generation by proposing Attentive-GP as a combination of the Transformer architecture \cite{vaswani:etal:2017} and the GP regression. The input time-series is encoded into a feature space via a linear transformation layer and then combined with positional encoding. The attention mechanism searches the most relevant information across all time steps between the features of input and output sequences. The output layer is replaced by a GP regression layer to map the featured attentions to output sequence with a probabilistic Bayesian representation. The combination of a Transformer network and the GP regression poses a computational challenge. The training of the GP layer requires deterministic gradient using full-batch data, whereas the Transformer can be better trained using mini-batch data. Therefore, we propose a block-wise training algorithm with a convergence proof. The algorithm treats the Transformer and the GP layer as separate blocks and trains them alternatively for more efficient training. This algorithm can be applied to any hybrid models of neural networks with kernel machines, including GPs and support vector machines (SVMs). 

The contributions of this paper are 1) We propose a tractable Bayesian generative models for real-valued sequences with explicit distributions,and 2) The proposed block-wise training algorithm makes the Attentive-GP more scalable to large data sets.  

\section{Attentive-GP} \label{sec:methodology}
The new method is based on the combination of Transformer architecture and GP regression. The last layer in the Transformer architecture is replaced by a GP regression layer $f(\cdot)$ to retain a probabilistic Bayesian representation. Meanwhile, the penultimate layer in the Transformer serves as a feature extractor $\phi(\cdot)$. The input sequence $\{\mathbf{x}_1,...,\mathbf{x}_L\}$ and the past output sequence $\{y_1,...,y_{i-1}\}$ are propagated through the feature extractor $\phi(\cdot)$ to get a feature vector $\bar{\mathbf{x}}_i$ that contains the most relevant information to predict next output $y_i$.
\begin{equation}
    \bar{\mathbf{x}}_{i} =\phi\left(\{\mathbf{x}_1,...,\mathbf{x}_L\}, \{y_1,...,y_{i-1}\}\right)
\end{equation}
Then $\bar{\mathbf{x}}_{i}$ is mapped to $y_i$ through the GP regression layer $f(\cdot)$ as follows
\begin{equation} \label{eq:regression}
    y_{i}  = f(\bar{\mathbf{x}}_i) + \epsilon_i
\end{equation}
where $\varepsilon_i \sim \mathcal{N}\left(0, \sigma^{2}\right)$ is Gaussian noise with zero mean and variance $\sigma^{2}$.

The structure of the proposed architecture is shown in Fig. \ref{fig:transformer_gp}. Details about $\phi(\cdot)$ and $f(\cdot)$ are introduced in the following subsections.

\begin{wrapfigure}{l}{0.6\textwidth}    
      \centering
      \includegraphics[width=1.0\linewidth]{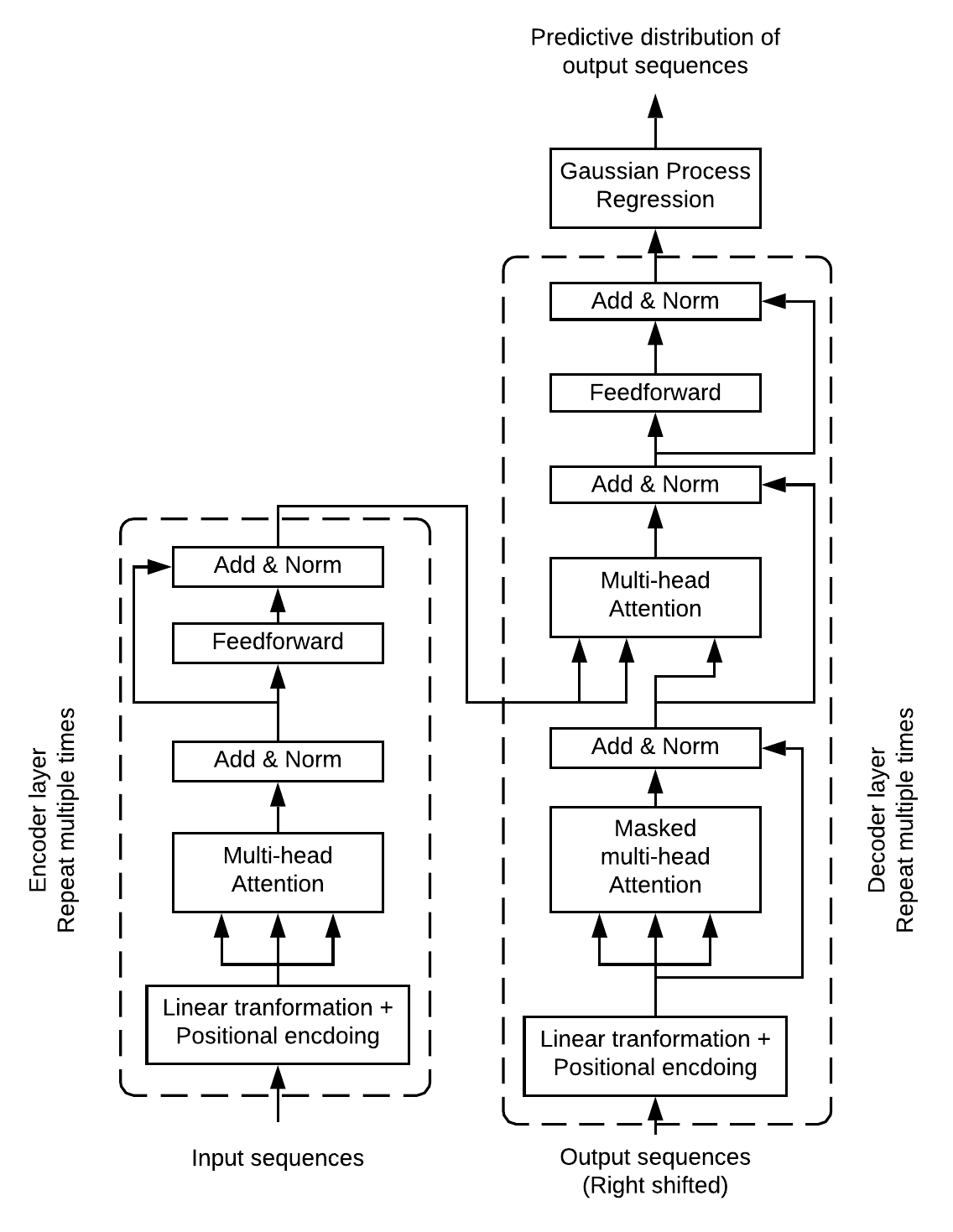}
      \caption{Model architecture of Attentive-GP}
      \label{fig:transformer_gp}
\end{wrapfigure}

\subsection{Encoder-decoder with attention}
The original input sequence $\{\mathbf{x}_1,...,\mathbf{x}_L\}$ is converted to a $d_x$-dimensional embedded sequence $\{\mathbf{x}_1^{'},...,\mathbf{x}_L^{'}\}$ via linear transformation and sinusoidal positional encoding. To accelerate the training process, RNNs are not used to encode the input and output sequences here. Therefore, positional encoding, such as a $\operatorname{sin}$ or $\operatorname{cos}$ function \cite{gehring:etal:2017,sukhbaatar:etal:2015}, has to be added to the input and output sequences to expose the position information to the model.

Given $\{\mathbf{x}_1^{'},...,\mathbf{x}_L^{'}\}$, where $\mathbf{x}_i^{'} \in \mathbb{R}^{d_x}$, an attention layer in the encoder computes a sequence $\{\mathbf{h}_1,...,\mathbf{h}_L\}$, where $\mathbf{h}_i \in \mathbb{R}^{d_h}$. An attention mechanism maps a query and a set of key-value pairs to a weighted sum of the values \cite{shaw:etal:2018}. The query, key and value can be vectors computed from linearly transformed elements in input or output sequences. Let $\mathbf{W}^{Qx}, \mathbf{W}^{Kx}, \mathbf{W}^{Vx} \in \mathbb{R}^{d_{x} \times d_{h}}$ be parameters in the attention layer within the encoder. First, the compatibility $e_{ij}$ between two input elements $\mathbf{x}_i^{'}$ and $\mathbf{x}_j^{'}$ is computed by scaled dot product \cite{vaswani:etal:2017}
\begin{equation} \label{eq:mha_compatibility}
    e_{i j}=\frac{\left(\mathbf{x}_{i}^{'} \mathbf{W}^{Qx}\right)\left(\mathbf{x}_{j}^{'} \mathbf{W}^{Kx}\right)^{T}}{\sqrt{d_{h}}}
\end{equation}
Then the attention weight $\alpha_{ij}$ is computed using softmax function
\begin{equation} \label{eq:attention_weight}
    \alpha_{i j}=\frac{\exp (e_{i j})}{\sum_{k=1}^{L} \exp (e_{i k})}
\end{equation}
Consequently, $\mathbf{h}_i$ is calculated as weighted sum of value vectors across all elements in the sequence
\begin{equation} \label{eq:attention}
    \mathbf{h}_{i}=\sum_{j=1}^{L} \alpha_{i j}\left(\mathbf{x}_{j} \mathbf{W}^{Vx}\right)
\end{equation}
In addition, $\mathbf{h}_i$ is passed through a feedforward layer to get encoder output $\mathbf{c}_i$,  where $\mathbf{c}_i \in \mathbb{R}^{d_c}$.

Two attention layers are used in the decoder. Similar to Eqs. \ref{eq:mha_compatibility}, \ref{eq:attention_weight} and \ref{eq:attention}, the first attention layer of the decoder maps the output sequence $\{\mathbf{y}_1^{'},...,\mathbf{y}_L^{'}\}$, which is converted from the original sequence $\{y_1,...,y_L\}$ plus sinusoidal position encoding, to $\{\mathbf{s}_1,...,\mathbf{s}_N\}$, where $\mathbf{s}_i \in \mathbb{R}^{d_s}$. The second attention layer of the decoder takes $\{\mathbf{c}_1,...,\mathbf{c}_N\}$ and $\{\mathbf{s}_1,...,\mathbf{s}_N\}$ as inputs and finds correlated information $\mathbf{d}_i$ between the two input sequences. Finally, $\mathbf{d}_i$ is passed through a feedforward layer to get the feature vector $\bar{\mathbf{x}}_i \in \mathbb{R}^F$, that is to be used in a Gaussian process regression layer to compute the distribution of $y_i$.

\subsection{Gaussian Process regression layer}
Following Eq. \ref{eq:regression}, observation $y_i$ is obtained by adding Gaussian noise to unobserved latent variable $f(\bar{\mathbf{x}}_i)$. Let $p(Y|\mathbf{f}) = \mathcal{N}(\mathbf{f}, \sigma^{2}I)$ denote the likelihood that relates the noisy observations $Y = [y_1, ..., y_N]$ to latent predicted values $\mathbf{f} = f(\bar{\mathbf{X}}) =\left[f\left(\bar{\mathbf{x}}_{1}\right), \ldots, f\left(\bar{\mathbf{x}}_{N}\right)\right]^{\top}$. Furthermore, the prior distribution of $\mathbf{f}$ is a multivariate Gaussian $p(\mathbf{f}) = \mathcal{N}(\boldsymbol{\mu}, K_{X, X})$ with mean $\boldsymbol{\mu}$ and covariances matrix $K_{X, X}$. Each element in the covariance matrix $(K_{X, X})_{i,j}$ is computed by the kernel function $\kappa(\mathbf{x}_i, \mathbf{x}_j)$ \cite{scholkopf:etal:2002}. The squared exponential kernel is chosen in this study because it can universally approximate any continuous function within an arbitrarily small epsilon band \cite{micchelli:etal:2006}. It has the form
\begin{equation}
    \kappa\left(\bar{\mathbf{x}}_{i}, \bar{\mathbf{x}}_{j}\right)=\exp \left(-\frac{1}{2}\left(\bar{\mathbf{x}}_{i}-\bar{\mathbf{x}}_{j}\right)^{\top} \Theta^{-2}\left(\bar{\mathbf{x}}_{i}-\bar{\mathbf{x}}_{j}\right)\right)
\end{equation}
where hyperparameter $\Theta$ is a lengthscale factor which will be learned along with the parameters of neural networks as part of training. The predictive distribution of $\mathbf{f}_*$ at test inputs $\bar{\mathbf{X}}_{*}$ can be induced from the training data,
\begin{equation}
\begin{aligned} 
& p(\mathbf{f}_{*} | \bar{\mathbf{X}}_{*}, \bar{\mathbf{X}}, \mathbf{y}) = \mathcal{N}\left(\mathbb{E}\left[\mathbf{f}_{*}\right], \operatorname{cov}\left(\mathbf{f}_{*}\right)\right) \\ 
& \mathbb{E}\left[\mathbf{f}_{*}\right] =\boldsymbol{\mu}_*+K_{X_{*}, X}\left[K_{X, X}+\sigma^{2} I\right]^{-1} (\mathbf{y} - \boldsymbol{\mu}) \\ 
& \operatorname{cov}\left(\mathbf{f}_{*}\right) =K_{X_{*}, X_{*}}-K_{X_{*}, X}\left[K_{X, X}+\sigma^{2} I\right]^{-1} K_{X, X_{*}} 
\end{aligned}
\end{equation}
where $\mathbb{E}\left[\mathbf{f}_{*}\right]$ is the mean of predicted outputs, and the diagonal elements in $\operatorname{cov}\left(\mathbf{f}_{*}\right)$ is the variance of predicted outputs \cite{Rasmussen:Williams:2006}. 

\section{Block-wise training}
Due to the probabilistic nature of the GP regression layer, the negative log marginal likelihood function is used as the loss function for the proposed method. Let $\mathbf{W}$ be parameters in the attention layers $\phi(\cdot)$ and $\boldsymbol{\theta} = [\Theta, \sigma^2]$ be parameters in the GP layer $f(\cdot)$. The negative log marginal likelihood has the following form \cite{Rasmussen:Williams:2006}:
\begin{equation} \label{eq:nll}
    \begin{aligned} 
        \mathcal{L} & = - \log p(\mathbf{y} | \bar{\mathbf{X}}, \boldsymbol{\theta}, \mathbf{W}) \\ 
        & = \frac{1}{2}(\mathbf{y}-\boldsymbol{\mu})^{T}\left(K_{X,X}+\sigma^{2} I\right)^{-1}(\mathbf{y}-\boldsymbol{\mu}) +\frac{1}{2} \log \left|K_{X,X}+\sigma^{2} I\right|+\frac{N}{2} \log 2 \pi
    \end{aligned}
\end{equation}
where $K_{X,X}$ implicitly depends on $\Theta$ and $\mathbf{W}$. The proposed model is trained by minimizing the loss function with respect to $\mathbf{W}$ and $\boldsymbol{\theta}$. It is important to note that $K_{X,X}$ and its inverse in the objective function must be calculated using the entire training data set. Therefore, the gradient with respect to $\boldsymbol{\theta}$ needs to be calculated using the entire training data set, precluding the mini-batch based approaches as a training method for the proposed model.

It is technically possible to train the proposed model using the full-batch gradient descent algorithm without factorizing the training data set \cite{wilson:etal:2016}. However, it requires a large amount of memory in the training process and hence not scalable to large training data. A more scalable approach is to use a hybrid approach, where $\mathbf{W}$ is updated by mini-batch data and $\boldsymbol{\theta}$ is updated by full-batch data asynchronously. Notice also that it is not desirable to update $\boldsymbol{\theta}$ in the early stage of training process because the input to GP layer $\bar{\mathbf{x}}_i$ is not stable while $\mathbf{W}$ is not converged. 

\subsection{Algorithm}
We propose a simple and effective algorithm that trains the proposed model in a block-wise manner as described in Algorithm \ref{algorithm}, where the model parameters are divided into two blocks $\mathbf{W}$ for the neural network and $\boldsymbol{\theta}$ for the GP layer. The algorithm updates $\mathbf{W}$ by stochastic gradients using mini-batch data, and $\boldsymbol{\theta}$ by deterministic gradient using full-batch data. Both blocks of parameters are repeatedly updated for a few (stochastic) gradient steps alternately until convergence.

\begin{algorithm}[h] 
  \caption{Block-wise training algorithm}
  \label{alg:example}
\begin{algorithmic} \label{algorithm}
  \STATE Initialize parameters $\boldsymbol{\theta}_0^0, \mathbf{W}_0^0$
  \FOR {$k = 1, 2, 3,...$}
      \FOR {$\tau = 1, 2,...T_1$}
      
      \STATE Compute stochastic gradient $\mathbf{g}(\mathbf{W}_{k-1}^{\tau-1}, \xi_k^{\tau})$ based on mini-batch data $\xi_k^{\tau}$

      \STATE $\displaystyle \mathbf{W}_{k-1}^{\tau} = \mathbf{W}_{k-1}^{\tau-1} - \eta_{k\tau}^{\mathbf{W}}\mathbf{g}(\mathbf{W}_{k-1}^{\tau-1}, \xi_k^{\tau})$

      \ENDFOR

      $\displaystyle \mathbf{W}_k^0 \leftarrow \mathbf{W}_{k-1}^{T}$

      \FOR {$\tau = 1, 2,...T_2$}

      \STATE Compute gradient $\mathbf{g}(\boldsymbol{\theta}_{k-1}^{\tau-1})$ based on full-batch training data

      \STATE $\displaystyle \boldsymbol{\theta}_{k-1}^{\tau} = \boldsymbol{\theta}_{k-1}^{\tau-1} - \eta_{k\tau}^{\boldsymbol{\theta}}\mathbf{g}(\boldsymbol{\theta}_{k-1}^{\tau-1})$

      \ENDFOR
  
      $\displaystyle \boldsymbol{\theta}_k^0 \leftarrow \boldsymbol{\theta}_{k-1}^{T}$
  \ENDFOR
  \STATE Return $\boldsymbol{\theta}, \mathbf{W}$
\end{algorithmic}
\end{algorithm}

The memory requirement in training becomes significant when neural network architecture in $\phi(\cdot)$ contains a large number of parameters, such as the Transformer. The block-wise algorithm could train longer sequences than that in the full-batch settings.  For scalability of the GP layer, the kernel matrix $K_{X, X}$ is approximated by the KISS-GP \cite{wilson:Nickisch:2015} covariance matrix
\begin{equation}
    K_{X, X} \approx K_{KISS} = SK_{U, U}S^{\top}
\end{equation}
where $S$ is a sparse matrix of interpolation weights and $K_{U, U}$ is the covariance matrix evaluated over $u$ latent inducing points. Since $S$ is sparse and $K_{U, U}$ is structured, this approximation makes complexity of training GP layer by full batch data be $\mathcal{O}(N)$.
 
\subsection{Convergence analysis}
We prove that the block-wise training algorithm converges to a fixed point. The convergence guarantee is based on sufficient descent in loss per round of block-wise training as formally stated in Lemma \ref{lemma:sufficient_descent}. 
\begin{lemma} \label{lemma:sufficient_descent}
    \begin{equation*}
        \begin{aligned}
            & \mathbb{E}[\mathcal{L}(\mathbf{W}_{k}^{0}, \boldsymbol{\theta}_{k}^{0})] - \mathbb{E}[\mathcal{L}(\mathbf{W}_{k-1}^{0}, \boldsymbol{\theta}_{k-1}^{0} )] \\
            \leq & -\frac{1}{2}\sum_{\tau=0}^{T_1}\eta_{k\tau}^{\boldsymbol{\theta}} \mathbb{E}[\|\nabla_{\boldsymbol{\theta}} \mathcal{L}(\mathbf{W}_{k}^{0}, \boldsymbol{\theta}_{k-1}^{\tau} )\|^2]  -\frac{1}{2}\sum_{\tau=0}^{T_2}\eta_{kt}^{\mathbf{W}}\mathbb{E}[ \|\nabla_{\mathbf{W}} \mathcal{L}(\mathbf{W}_{k-1}^{\tau}, \boldsymbol{\theta}_{k-1}^0 )\|^2]  + \frac{1}{2}\mathbb{LM}\sum_{\tau=0}^{T_1} ({\eta_{k\tau}^{\mathbf{W}}})^2
        \end{aligned}
    \end{equation*}
    where $\mathbb{L}$ is the Lipschitz constant for the loss function and $\mathbb{M}$ is a non-negative scalar.
\end{lemma}
The sufficient descent property leads to the convergence property of the proposed block-wise training algorithm as in Theorem \ref{theom1}.
\begin{theom} \label{theom1}
    Let $\{\mathbf{W}_k^0, \boldsymbol{\theta}_k^0\}$ be a sequence generated by the proposed training algorithm. The limit of this sequence converges to a stationary point where $\mathbb{E} [ \|\nabla\mathcal{L}(\mathbf{W}_k^0, \boldsymbol{\theta}_k^0) \|] = 0$ as $k \rightarrow \infty$.
\end{theom}
The details of proof can be found in the supplementary materials. Note that the stationary point obtained by the block-wise training algorithm may not be the same as that in the full batch training algorithm, but empirical studies show that block-wise training algorithm converges to similar results faster than the full batch algorithm (see Fig. \ref{fig:train_loss}). In addition, the convergence property of the proposed training algorithm does not rely on any assumptions about the architecture of neural networks. Therefore, the proposed training algorithm can be applied to any hybrid models of neural networks with kernel machines, such as SVMs.

\section{Experiments} \label{sec:experiments}
\subsection{Sequence data}
In this study, three cases are used to demonstrate the effectiveness and reliability of the proposed Attentive-GP approach. The details of the selected data sets are described below.

\textbf{Robot arm\footnote{http://www.gaussianprocess.org/gpml/data/}} The data is generated from a seven-degree-of-freedom anthropomorphic robot arm, and collected at 100Hz from an actual robot performing various rhythmic movements. The input data consists of 21 dimensions, including positions, velocities and accelerations at the seven joins of the robot arm. The goal is to learn the dynamics of the torque of the robot arm given a sequence of 21-dimensional inputs. Learning the dynamics of the robot arm is challenging due to strong nonlinearity and measurement noise. A good generative model for robotic dynamics enables stochastic simulation of robot movements and can be further used to manipulate the robot arm torque using model-based control.

\textbf{Suspension system\footnote{http://www.nonlinearbenchmark.org/}} The suspensions in motor vehicles are composed of shock absorbers and progressive springs. The data in this application is generated from a second order linear time-invariant system with a third degree polynomial static nonlinearity around it in feedback. The input is external excitation and the output is the movement in the suspension system within a motorcycle. The suspension system is a nonlinear and fast-responsive system. A good generative model that well characterizes such nonlinear dynamic systems can be used in many mechanical oscillating processes within vehicles. 

\textbf{Smart grid\footnote{https://www.kaggle.com/c/global-energy-forecasting-competition-2012-load-forecasting/data}} The smart grid data is a sequence of hourly temperature measurements at 11 different cities (hence, dimensionality of 11) in the USA from 2004 to 2008. The task is to forecast the hourly electricity load on the entire grid. Forecasting the load on the grid over next dozens of hours is important in the utility industry for electricity generation. The uncertainty related to load on grid can also be used by the utility industry towards transmission and distribution operations.

For each data set, 80\% of samples are used for training, and the remaining 20\% of samples are used for testing. In addition, the proposed method is compared against RNN, LSTM, GRU, RNN encoder-decoder and Transformer (Attention). The details about the model architecture and training setup are listed in Appendix.

\subsection{The performance of block-wise training}
The proposed block-wise training algorithm is compared against the full batch training algorithm. For a fair comparison we adjust $T_1$ and $T_2$ of Algorithm~\ref{algorithm} as (Sample Size/Batch Size) and $1$, respectively. This way, the computation per backpropation is comparable between the full-batch and the block-wise training algorithms since each training sample is used once within one epoch in computing gradient for each parameter. The learning rate for $\mathbf{W}$ is 0.001 and the learning rate for $\boldsymbol{\theta}$ is 0.01 for both algorithms in the Adam optimizer \cite{kingma:adam:2015}. The training process is repeated 10 times for each data set with random initial conditions. The comparison between the two algorithms are shown in Fig. \ref{fig:train_loss}. 

The block-wise training algorithm converges to similar results faster than the full-batch algorithm in all three cases. The final values of training loss by both algorithms are shown in Figs. \ref{fig:train_loss_robot}, \ref{fig:train_loss_silver} and \ref{fig:train_loss_grid}, which demonstrate that the block-wise training algorithm can find comparable or even better solutions with less computation. Faster convergence of block-wise training algorithm can be attributed to multiple updates on $\mathbf{W}$ per epoch, while full batch algorithm only updates $\mathbf{W}$ once per epoch. In conclusion, the proposed block-wise training algorithm not only converges in theory and practice, but also requires less resources for computation and memory than the full batch training algorithm under appropriate settings. 
\begin{figure*}[ht]
    \centering
    \subfigure[]{\includegraphics[width=4.5cm]{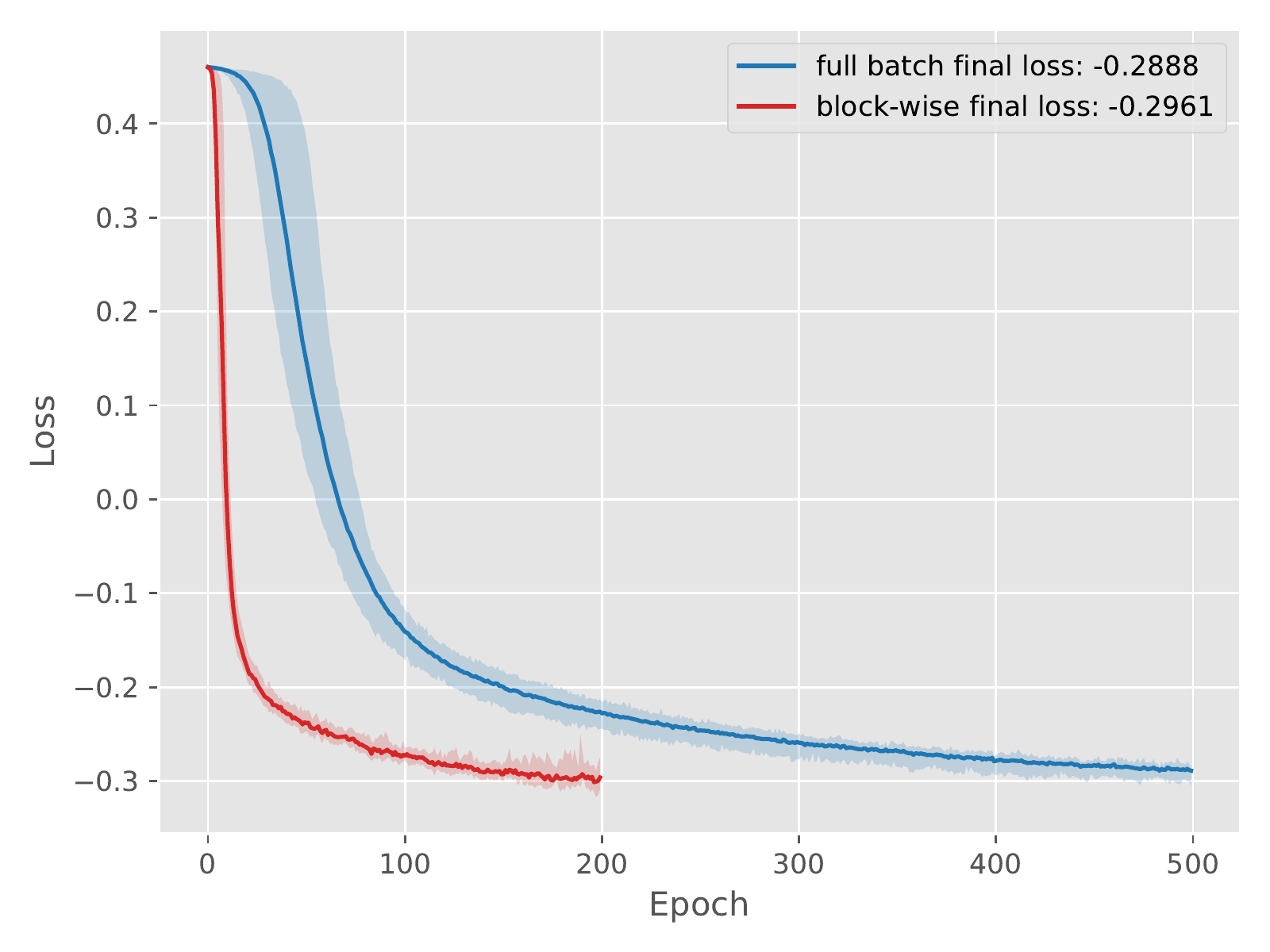} \label{fig:train_loss_robot}} 
    \subfigure[]{\includegraphics[width=4.5cm]{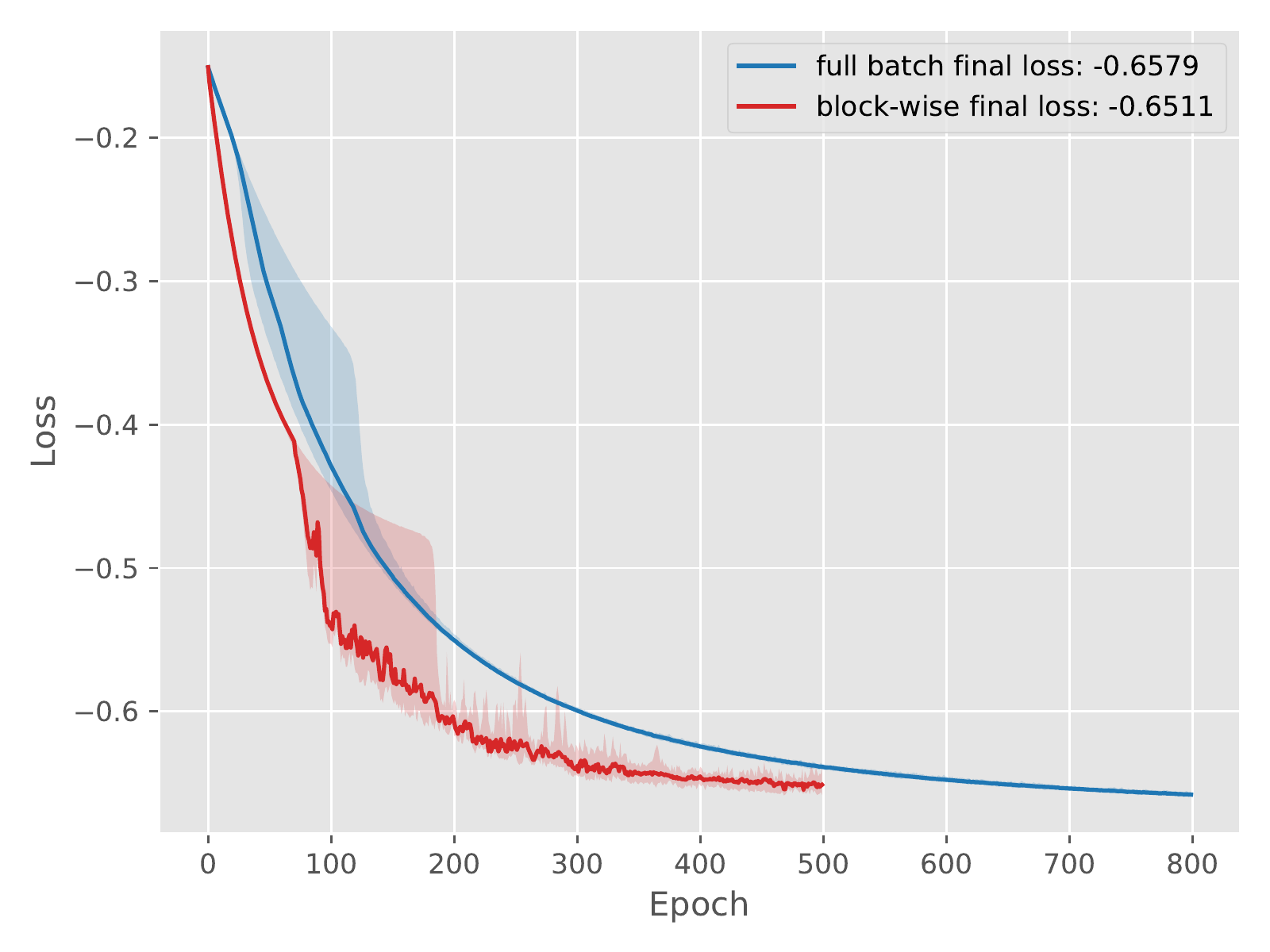} \label{fig:train_loss_silver}} 
    \subfigure[]{\includegraphics[width=4.5cm]{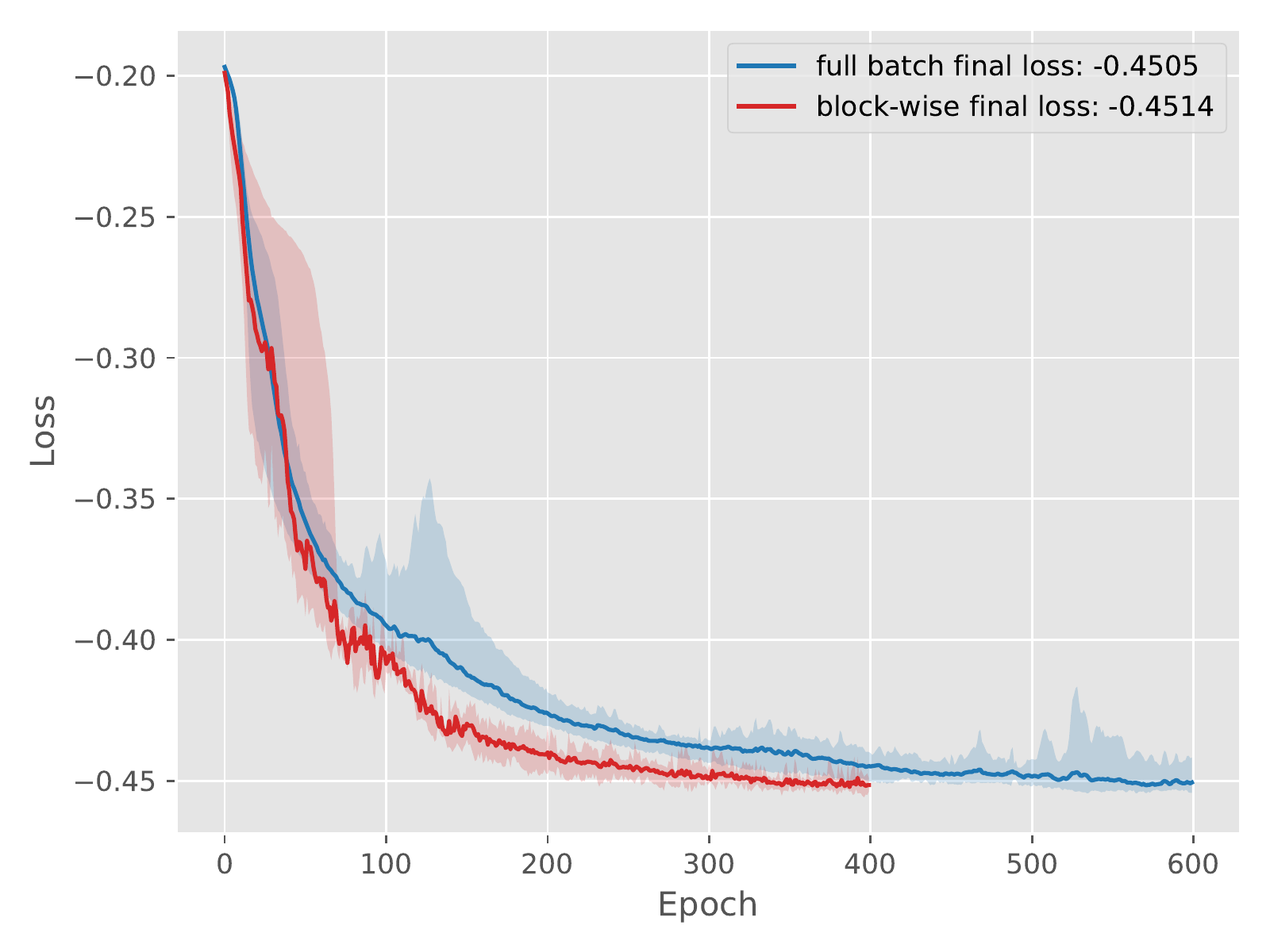} \label{fig:train_loss_grid}} 
    \caption{Comparison of loss (negative log marginal likelihood) between full batch and block-wise training algorithms. (a) Robot arm data; (b) Suspension position data; (c) Electricity load data.} \label{fig:train_loss}
\end{figure*}

\subsection{Accuracy comparison for prediction and generation tasks}

We compare the quality of sequences generated by our model against existing methods including encoder-decoder models and recurrence-based models. Currently, no universal criteria are available to evaluate generative models \cite{theis:etal:2016}. Unlike the generative models for language and images, it is less intuitive to compare generated time-series data. We measure the normalized root mean squared error (NRMSE) between the actual observed time-series in test data and the mean of the modelled generative distribution, based on the assumption that the average of samples should be close to the mean of the generative distribution.

The Attentive-GP learns the true generative distribution $p_{data}(\mathbf{y})$ of output sequences $\mathbf{y} = \{y_1, ..., y_L\}$ from which realized output sequences are to be sampled given input sequences. The generation task is performed by iteratively computing a new output $\hat{y}_t$ for $t=1,...,L$ given a previously {\em computed} output $\{\hat{y}_1,...,\hat{y}_{t-1}\}$ and the entire input sequence $\{\mathbf{x}_1,...,\mathbf{x}_L\}$ as inputs, whereas the prediction task is performed by predicting a new output $\hat{y}_t$ given the currently {\em observed} output $\{y_1,...,y_{t-1}\}$ and the entire input sequence $\{\mathbf{x}_1,...,\mathbf{x}_L\}$ as inputs. Therefore, accuracy of prediction task is significantly better than that of generation task across all methods considered. The NRMSE values for both generation and prediction tasks are reported in Table~\ref{tab:rmse}.

\begin{table}[ht]
    \begin{center}
        \begin{tabular}{| c | c | c | c | c | c | c |}
        \hline
             & \multicolumn{2}{c|}{Robot} & \multicolumn{2}{c|}{Suspension} & \multicolumn{2}{c|}{Grid} \\ 
        \cline{2-7}
               & Prediction                & Generation & Prediction & Generation & Prediction & Generation \\ \hline
RNN           & 1.3e-2\scriptsize{(7e-4)} & 4.9e-2\scriptsize{(9e-4)} & 1.2e-2\scriptsize{(5e-4)} & 4.4e-2\scriptsize{(8e-4)} & 3.5e-2\scriptsize{(6e-4)} & 9.7e-2\scriptsize{(8e-4)} \\ 
GRU           & 1.2e-2\scriptsize{(7e-4)} & 4.8e-2\scriptsize{(8e-4)} & 1.2e-2\scriptsize{(6e-4)} & 4.2e-2\scriptsize{(8e-4)} & 3.4e-2\scriptsize{(7e-4)} & 8.7e-2\scriptsize{(9e-4)} \\
LSTM          & 1.2e-2\scriptsize{(7e-4)} & 4.7e-2\scriptsize{(9e-4)} & 1.2e-2\scriptsize{(6e-4)} & 4.3e-2\scriptsize{(8e-4)} & 3.4e-2\scriptsize{(7e-4)} & 8.6e-2\scriptsize{(8e-4)} \\
\hline
RNN Enc-Dec   & 9.6e-3\scriptsize{(8e-4)} & 1.4e-2\scriptsize{(1e-3)} & 6.8e-3\scriptsize{(7e-4)} & 1.9e-2\scriptsize{(1e-3)} & 1.3e-2\scriptsize{(9e-4)} & 3.6e-2\scriptsize{(1e-3)} \\ 
Transformer   & 8.4e-3\scriptsize{(9e-4)} & 1.4e-2\scriptsize{(1e-3)} & 6.3e-3\scriptsize{(8e-4)} & 1.5e-2\scriptsize{(1e-3)} & 1.2e-2\scriptsize{(9e-4)} & 3.3e-2\scriptsize{(1e-3)} \\
Attentive-GP  & 5.5e-3\scriptsize{(1e-3)} & 1.1e-2\scriptsize{(1e-3)} & 6.1e-3\scriptsize{(1e-3)} & 1.3e-2\scriptsize{(1e-3)} & 1.2e-2\scriptsize{(9e-4)} & 3.2e-2\scriptsize{(1e-3)} \\
\hline
    \end{tabular}
    \end{center}  
    \caption{NRMSE of predicted and generated time-series. The NRMSE is calculated between the means of predictive/generative distributions and samples in the test data set.} \label{tab:rmse}
\end{table}

The encoder-decoder architectures (RNN Enc-Dec, Transformer and Attentive-GP) significantly outperform standalone RNNs and its variants such as LSTM and GRU. The encoder-decoder architectures learn the relationship between inputs and outputs across the entire sequence, while RNNs only learn the pair-wise relationship between the past sequence and the one-step ahead prediction. Considering that the dynamics in some time-series could be strongly nonlinear (e.g. robot arms),  such nonlinear dynamics cannot be easily modelled by predefined lag structure. The Transformer architecture leads to a better generation accuracy because the multi-head attention mechanism in Transformer is able to learn richer representation of sequences than the single-head attention mechanism in RNN encoder-decoder. Although the training data contains noise to increase the difficulty of learning, the proposed Attentive-GP approach achieves the best accuracy among all the methods. The superior performance of the proposed method can be attributed to the Bayesian framework within the GP layer, which naturally considers the noise in target sequences.

\begin{figure*}[ht]
    \centering
    \subfigure[]{\includegraphics[width=4.5cm]{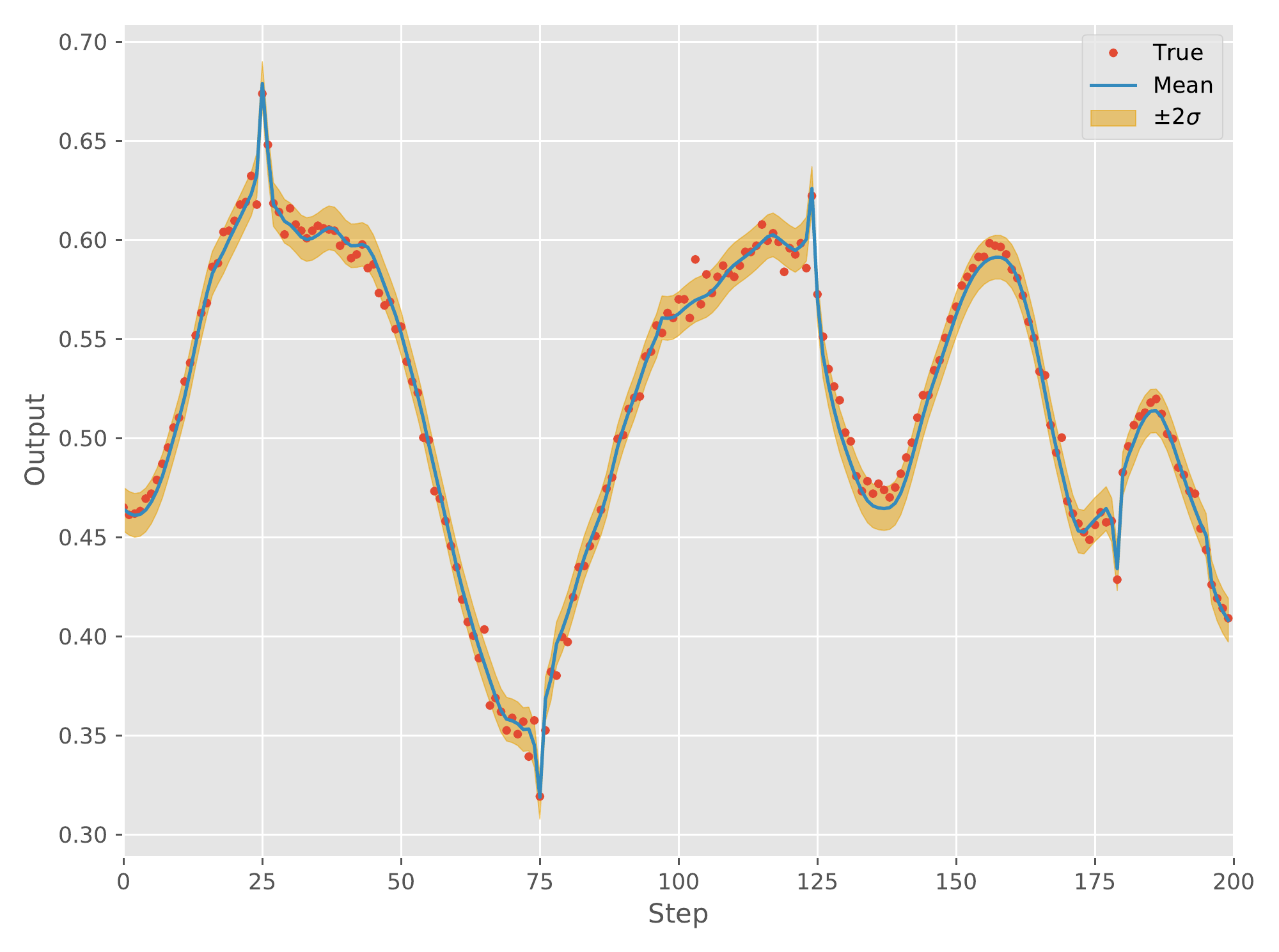} \label{fig:robot_generations}} 
    \subfigure[]{\includegraphics[width=4.5cm]{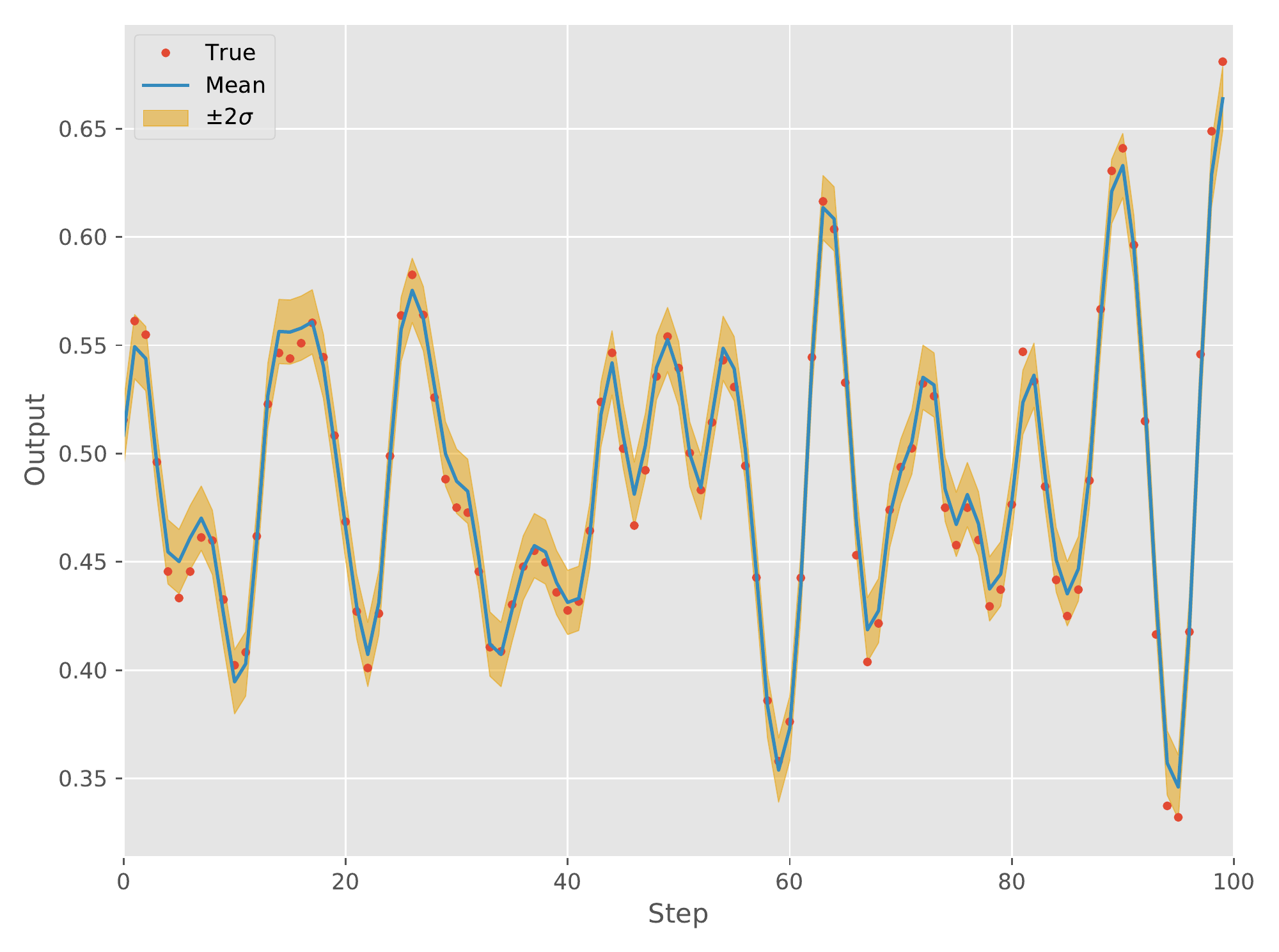} \label{fig:silver_generations}} 
    \subfigure[]{\includegraphics[width=4.5cm]{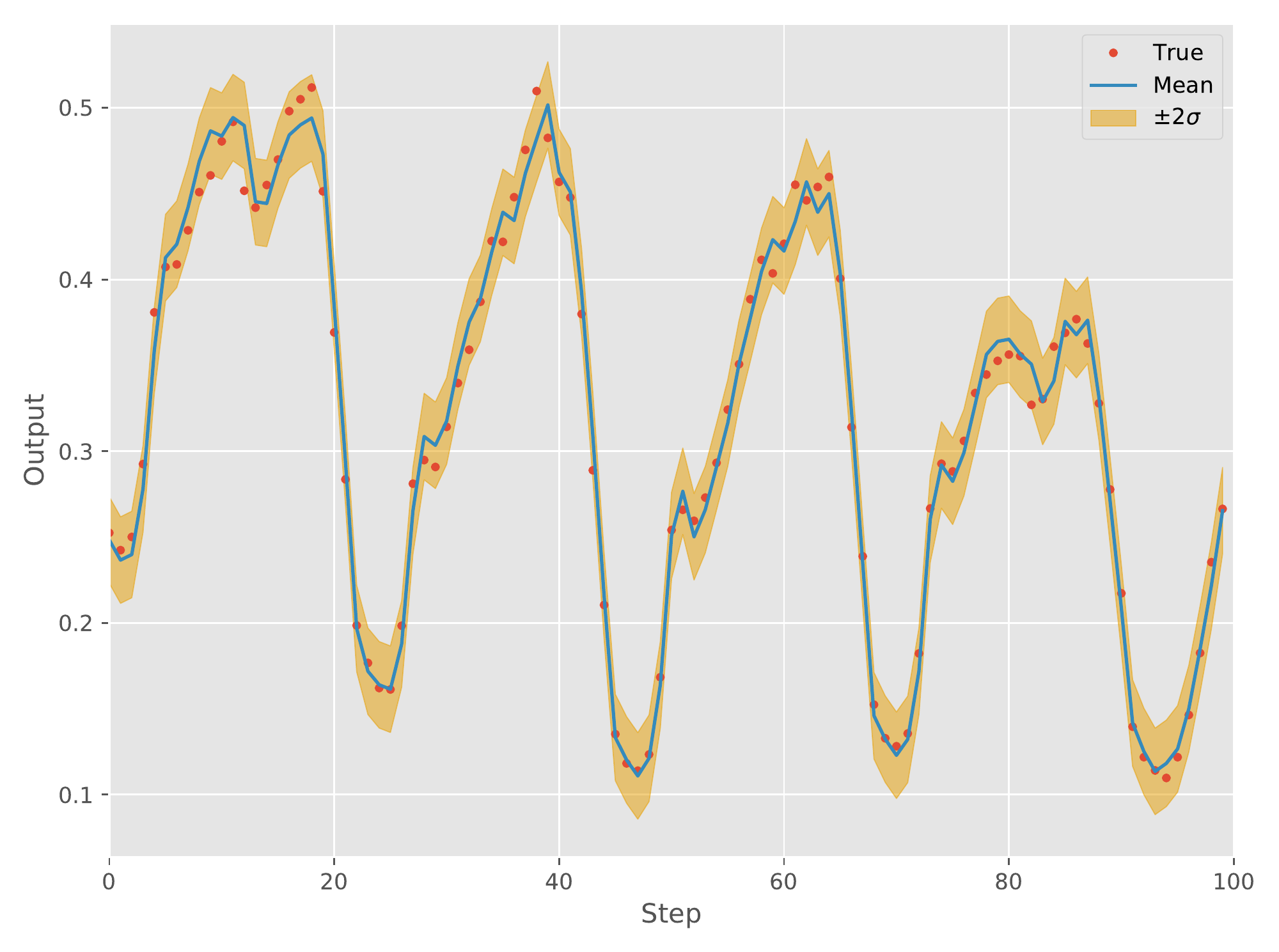} \label{fig:grid_generations}} 
    \caption{(a) Generated robot arm torque for 200 steps; (b) Generated suspension position for 100 steps; (c) Generated electricity load for 100 hours; The generative distribution is depicted with mean and shaded 2-sigma range. The data is normalized between [0, 1].}
\end{figure*}

The observed outputs in test data are plotted against the mean of the generative distribution along with $\pm 2\sigma$ interval in Figs. \ref{fig:robot_generations}, \ref{fig:silver_generations} and \ref{fig:grid_generations}. The fact that approximately 95\% of observed outputs are within the  $\pm 2\sigma$ interval indicates that the learned generative distribution from Attentive-GP is a good representation for the true generative distribution for output sequences. Sequences drawn from the learned generative distribution reflect the true dynamics in time-series and can be used for stochastic simulation.

\section{Related work} \label{sec:related}
The problem of generating time-series conditioned on external inputs is closely related to sequence-to-sequence learning. The encoder-decoders with attention mechanisms are the current state-of-the-art in sequence-to-sequence learning tasks \cite{Cho:etal:2014}. The attention mechanisms within encoder-decoders search for the most relevant information from the input and output sequences to predict the future outputs \cite{bahdanau:etal:2015}. Traditional encoder-decoders employ RNNs to embed the input and output sequences \cite{sutskever:etal:2014}, but the recurrent structure precludes the massive parallel computation on GPUs. Recently, the Transformer architecture is developed by computing dot-product attention multiple times (multi-head attention) directly on input and output sequences without RNN embedding, leading to significantly improved efficiency and accuracy \cite{vaswani:etal:2017}. However, the aforementioned encoder-decoders do not compute complete distributions for real-valued time-series. 

Recently, there has been some work on incorporating attention mechanisms into neural processes for probabilistic time-series prediction \cite{kim:etal:2019}. Inspired by GPs, neural processes learn to model distributions over functions via neural networks, and are able to estimate the uncertainty over their predictions \cite{garnelo:etal:2018a,garnelo:etal:2018b}. Under certain conditions, neural processes are mathematically equivalent to GPs \cite{rudner:etal:2018}, but it is not plausible to compare them directly because they are trained on different training regimes. In addition, they suffer a fundamental drawback of under-fitting \cite{kim:etal:2019}. GPs are chosen in this study because GPs are consistent stochastic processes with closed form variance functions.

We consider the combination of GPs and neural networks in our model because such combination leads to expressive model architecture with probabilistic outputs. From the perspective of deep learning, GP has been utilized as the last output layer in feedforward, convolutional and recurrent neural networks to regress sequences to reals with quantified uncertainty \cite{wilson:etal:2016,calandra:etal:2016}. On the other hand, combinations of GPs and neural networks are extensions of GPs with nonlinear warped inputs \cite{snelson:2004,lazaro:2012}. Although the hybrid model of RNNs and GPs can predict time-series probabilistically over long-horizons by iteratively feeding predicted outputs as inputs, it may not capture the true relationship between sequences due to the user-specified lag structure in time-series. The proposed method differs from prior work by utilizing GP as the last layer in encoder-decoder architecture to generate output sequences with uncertainty based on the most relevant information returned by the multi-head attention mechanism.

\section{Conclusions} \label{sec:conclusion}
A new encoder-decoder architecture for probabilistic time-series generation is proposed in this paper. The multi-head attention based encoder-decoder with a GP output layer, termed Attentive-GP, has strong feature extraction capability, while retaining the probabilistic Bayesian nonparametric representation. The proposed method outperforms a range of alternative approaches on sequence-to-sequence generation tasks. The Attentive-GP not only works on data with low to high noise levels, but also is scalable to time-series with different length. In short, the Attentive-GP provides a natural mechanism for Bayesian encoder-decoder, quantifying distributions for sequences while harmonizing with the neural networks based encoder-decoder.

The proposed block-wise training scheme can train the Attentive-GP model efficiently and effectively. With proved convergence property, this training algorithm is applicable to any hybrid models of neural networks with kernel machine output layers. An exciting direction for future research is to apply Attentive-GP in model-based control problems. 

\bibliography{../bib/abbreviations,../bib/articles,../bib/proceedings,../bib/books}
\bibliographystyle{plain}

\appendix

\section{Model Architecture}\label{sec:model_architecture}
\begin{table}[H]
    \begin{center}
        \begin{tabular}{ c  c  c  c }
        \hline
                            & Robot      & Suspension & Grid \\ 
        \hline 
        Linear Embedding    & 64         & 16         & 64    \\
        Attention Input Size & 64        & 32         & 64    \\
        Attention Layers    & 2          & 2          & 2     \\ 
        Attention Heads     & 2          & 2          & 2     \\
        Key Dimension       & 8          & 8          & 8     \\
        Query Dimension     & 8          & 8          & 8     \\
        Feature Extractor Output Size       & 4    & 2     & 4     \\
        \hline
        Sequence Length     & 100        & 100        & 24    \\
        Batch Size          & 1024       & 1024       & 64    \\
        Learning Rate $\mathbf{W}$  & 0.001 & 0.001  & 0.001  \\
        Learning Rate $\boldsymbol{\theta}$ & 0.01  & 0.01 & 0.01 \\
        Initial Noise in GP Layer  & 0.1 & 0.1 & 0.1 \\
        \hline
        \end{tabular}
    \end{center}  
    \caption{Model Architecture} \label{tab:model_architecture}
\end{table}

\section{Supplementary Material}\label{sec:supplementary_material}
Training the proposed Attentive-GP is a nonconvex optimization problem in the form of 
\begin{equation}
    \min_{\mathbf{W}, \boldsymbol{\theta}} \mathcal{L}(\mathbf{W}, \boldsymbol{\theta})
\end{equation}
where the decision variables are divided into two blocks $\mathbf{W}$ and $\boldsymbol{\theta}$. The proposed training algorithm is a special case of block coordinate descent algorithms because it minimizes $\mathcal{L}$ alternately between $\mathbf{W}$ and $\boldsymbol{\theta}$ while fixing the the other block at its last updated value \citeApp{wright:2015}. 

The convergence property of the block descent algorithms has been studied for convex objective functions \citeApp{beck2015,xu:yin2013}. Recently, the convergence property of the block descent algorithms for non-convex optimization problems has been analyzed  when all blocks are updated deterministically \citeApp{xu:yin:2017}. Their conclusion cannot be directly applied to our algorithm because $\mathbf{W}$ is updated by stochastic gradient and $\boldsymbol{\theta}$ is updated by deterministic gradient. Meanwhile, some attempts have been made to show the convergence analysis for stochastic block coordinate descent algorithms \citeApp{xu:yin:2015}. However, the decision variables in each block are updated by proximal stochastic gradient, not by pure stochastic gradient.

The convergence property of the proposed training algorithm is analyzed upon the ideas from \citeApp{beck2015} and \citeApp{xu:yin:2015}. There are several notable difference between the their settings and the one considered in this paper.
\begin{itemize}
    \item \citeApp{beck2015} and \citeApp{xu:yin:2015} consider an objective function that can be separated into continuously differentiable loss and (possibly) non-differentiable regularization. Our objective function does not contain any non-differentiable part.
    \item An additional proximal gradient update step is taken to tackle the (possibly) non-differentiable regularization in \citeApp{beck2015} and \citeApp{xu:yin:2015}. Pure (stochastic) gradient descent is taken in our algorithm. 
    \item The gradients in \citeApp{beck2015} for all blocks are computed deterministically, and the gradients in \citeApp{xu:yin:2015} for all blocks are stochastic. In our training algorithm, $\boldsymbol{\theta}$ is updated deterministically based on the entire data set, while $\mathbf{W}$ is is updated with stochastic gradient based on samples from a mini-batch. 
    \item In \citeApp{al:etal:2017,beck2015,xu:yin:2015}, each block descent step only includes one gradient descent step. In our algorithm, each block descent step contains multiple gradient descent steps.
\end{itemize}

We prove that the proposed training algorithm converges to a stationary point where $\mathbb{E} [ \|\nabla\mathcal{L}(\mathbf{W}_k^0, \boldsymbol{\theta}_k^0) \|] =0$.

\begin{assumption} \label{assumption:gradient_continuous}
    The derivatives of $\mathcal{L}$ w.r.t. $\mathbf{W}$ and $\boldsymbol{\theta}$ are uniformly Lipschitz with constant $L>0$
    \begin{equation}
        \begin{aligned}
            \left\|\nabla_{\boldsymbol{\theta}} \mathcal{L}(\mathbf{W}, \boldsymbol{\theta})-\nabla_{\boldsymbol{\theta}} \mathcal{L}(\mathbf{W}, \tilde{\boldsymbol{\theta}})\right\| & \leq L\|\boldsymbol{\theta}-\tilde{\boldsymbol{\theta}}\|  \\
            \quad\left\|\nabla_{\mathbf{W}} \mathcal{L}(\mathbf{W}, \boldsymbol{\theta})-\nabla_{\mathbf{W}} \mathcal{L}(\tilde{\mathbf{W}}, \boldsymbol{\theta})\right\| &\leq L\|\mathbf{W}-\tilde{\mathbf{W}}\|
        \end{aligned}
    \end{equation}
\end{assumption}
\begin{assumption} \label{assumption:gradient_variance}
    There exists scalars $M \geq 0$ and $M_G \geq 1$ such that, 
    \begin{equation}
        \mathbb{E}[\|\mathbf{g}(\mathbf{W}_{k}^{\tau}, \xi_k^{\tau})\|^2] \leq M + M_G  \|\nabla_{\mathbf{W}} \mathcal{L}(\mathbf{W}_{k}^{\tau}, \boldsymbol{\theta}_{k}^0)\|^2
    \end{equation}
\end{assumption}
\begin{assumption} \label{assumption:robbins_monro}
    The learning rates $0 < \eta^{\mathbf{W}}_k\tau < 1/LM_G$ and $0 < \eta^{\boldsymbol{\theta}}_k\tau < 1/LM_G$ are non increasing and satisfy the Robbins-Monre condition
    \begin{equation}
        \begin{aligned}
            \sum_{k=0}^{+\infty}\sum_{\tau=0}^{T_1} \eta_{k\tau}^{\mathbf{W}} & = \infty \\
            \sum_{k=0}^{+\infty}\sum_{\tau=0}^{T_2} \eta_{k\tau}^{\boldsymbol{\theta}} & = \infty \\
            \sum_{k=0}^{+\infty}\sum_{\tau=0}^{T_1} (\eta_{k\tau}^{\mathbf{W}})^2 & < \infty \\
            \sum_{k=0}^{+\infty}\sum_{\tau=0}^{T_2} (\eta_{k\tau}^{\boldsymbol{\theta}})^2 & < \infty
        \end{aligned}
    \end{equation} 
\end{assumption} 

We first show that the objective function has sufficient descent after each round of block-wise descent.

\begin{lemma} 
    \begin{equation}
        \begin{aligned}
            & \mathbb{E}[\mathcal{L}(\mathbf{W}_{k}^{0}, \boldsymbol{\theta}_{k}^{0})] - \mathbb{E}[\mathcal{L}(\mathbf{W}_{k-1}^{0}, \boldsymbol{\theta}_{k-1}^{0} )] \\
            \leq & -\frac{1}{2}\sum_{\tau=0}^{T_2}\eta_{k\tau}^{\boldsymbol{\theta}} \mathbb{E}[\|\nabla_{\boldsymbol{\theta}} \mathcal{L}(\mathbf{W}_{k}^{0}, \boldsymbol{\theta}_{k-1}^{\tau} )\|^2] \\
            &  -\frac{1}{2}\sum_{\tau=0}^{T_1}\eta_{kt}^{\mathbf{W}}\mathbb{E}[ \|\nabla_{\mathbf{W}} \mathcal{L}(\mathbf{W}_{k-1}^{\tau}, \boldsymbol{\theta}_{k-1}^0 )\|^2] \\
            & + \frac{1}{2}LM\sum_{\tau=0}^{T_1} ({\eta_{k\tau}^{\mathbf{W}}})^2
        \end{aligned}
    \end{equation}
\end{lemma}

\begin{proof}
    $\mathcal{L}$ is continuously differentiable because the squared exponential kernel function is infinitely differentiable and the loss function is differentiable w.r.t. to neural network weights. With Assumption \ref{assumption:gradient_continuous}, the following inequalities can be obtained using the second order Taylor series of  $\mathcal{L}(\mathbf{W}_{k}^{\tau-1}, \boldsymbol{\theta}_{k}^0 )$.
    \begin{equation} \label{eq:w_sufficient_descent_stochastic}
        \begin{aligned}
            \mathcal{L}(\mathbf{W}_{k}^{\tau}, \boldsymbol{\theta}_{k}^{0}, ) &  \leq \mathcal{L}(\mathbf{W}_{k}^{\tau-1}, \boldsymbol{\theta}_{k}^{0} ) \\
            & \quad + \nabla_{\mathbf{W}} \mathcal{L}(\mathbf{W}_{k}^{\tau-1}, \boldsymbol{\theta}_{k}^0)^{\top}(\mathbf{W}_{k}^{\tau} - \mathbf{W}_{k}^{\tau-1}) \\ 
            & \quad + \frac{L}{2}\|\mathbf{W}_{k}^{\tau} - \mathbf{W}_{k}^{\tau-1}\|^2 \\
            & = \mathcal{L}(\mathbf{W}_{k}^{\tau-1}, \boldsymbol{\theta}_{k}^{0} ) \\
            & \quad - \eta_{kt}^{\mathbf{W}}\nabla_{\mathbf{W}} \mathcal{L}(\mathbf{W}_{k}^{\tau-1}, \boldsymbol{\theta}_{k}^0 )^{\top}\mathbf{g}(\mathbf{W}_{k-1}^{\tau-1}, \xi_k^{\tau}) \\
            & \quad +\frac{{\eta_{k\tau}^{\mathbf{W}}}^2L}{2} \|\mathbf{g}(\mathbf{W}_{k-1}^{\tau-1}, \xi_k^{\tau})\|^2
        \end{aligned}  
    \end{equation}
    The desired bound is obtained by taking expectation w.r.t. $\xi_k^{\tau}$. In addition, $\mathbf{W}_{k}^{\tau}$ depends on $\xi_k^{\tau}$, but $\mathbf{W}_{k}^{\tau-1}$ does not. 
    
    The bound for $\mathcal{L}(\mathbf{W}_{k}^0, \boldsymbol{\theta}_{k}^{\tau})$ can be obtained in a more straightforward way because the gradient $\mathbf{g}(\boldsymbol{\theta}_{k}^{\tau-1}) = \nabla_{\boldsymbol{\theta}} \mathcal{L}(\mathbf{W}_{k}^0, \boldsymbol{\theta}_{k}^{\tau-1})$ is deterministic
\begin{equation} \label{eq:theta_sufficient_descent_deterministic}
    \begin{aligned}
        \mathcal{L}(\mathbf{W}_{k}^0, \boldsymbol{\theta}_{k}^{\tau}) &  \leq \mathcal{L}(\mathbf{W}_{k}^0, \boldsymbol{\theta}_{k}^{\tau-1} ) \\
        & \quad + \nabla_{\boldsymbol{\theta}} \mathcal{L}(\mathbf{W}_{k+1}, \boldsymbol{\theta}_{k})^{\top}(\boldsymbol{\theta}_{k+ 1} - \boldsymbol{\theta}_{k}) \\
        & \quad + \frac{L}{2}\|\boldsymbol{\theta}_{k+ 1} - \boldsymbol{\theta}_{k}\|^2 \\
        & = \mathcal{L}(\mathbf{W}_{k}^0, \boldsymbol{\theta}_{k}^{\tau-1} ) \\
        & \quad - \eta_{k\tau}^{\boldsymbol{\theta}} \|\nabla_{\boldsymbol{\theta}} \mathcal{L}(\mathbf{W}_{k}^{0}, \boldsymbol{\theta}_{k}^{\tau-1} )\|^2 \\
        & \quad + \frac{{\eta_{k\tau}^{\boldsymbol{\theta}}}^2L}{2} \|\nabla_{\boldsymbol{\theta}} \mathcal{L}(\mathbf{W}_{k}^{0}, \boldsymbol{\theta}_{k}^{\tau-1} )\|^2\\
        & \leq \mathcal{L}(\mathbf{W}_{k}^0, \boldsymbol{\theta}_{k}^{\tau-1} ) -\frac{\eta_{k\tau}^{\boldsymbol{\theta}}}{2} \|\nabla_{\boldsymbol{\theta}} \mathcal{L}(\mathbf{W}_{k}^{0}, \boldsymbol{\theta}_{k}^{\tau-1} )\|^2
    \end{aligned}  
\end{equation}
The objective function yielded in each (stochastic) gradient step is bounded by a quantity. It is clear that the right hand side of Eq. \ref{eq:theta_sufficient_descent_deterministic} is bounded by a deterministic quantify. The goal is to find a deterministic quantity to bound the right hand side of Eq. \ref{eq:w_sufficient_descent_stochastic}. Additional assumptions (Assumption \ref{assumption:gradient_variance}) on the second moment of the stochastic gradient $\mathbf{g}(\mathbf{W}_{k-1}^{\tau-1}, \xi_k^{\tau})$ are required to restrict the last term in Eq. \ref{eq:w_sufficient_descent_stochastic}. Assumption \ref{assumption:gradient_variance} is valid because the variance of stochastic gradient is bounded \citeApp{Jin:2019}.

Based on Assumption \ref{assumption:gradient_variance}
    \begin{equation*}
        \begin{aligned}
        & \mathbb{E}[\mathcal{L}(\mathbf{W}_{k}^{\tau}, \boldsymbol{\theta}_{k}^0)] - \mathcal{L}(\mathbf{W}_{k}^{\tau-1}, \boldsymbol{\theta}_{k}^0 ) \\
        & \leq -(1-\frac{1}{2}\eta_{kt}^{\mathbf{W}}LM_G)\eta_{kt}^{\mathbf{W}} \|\nabla_{\mathbf{W}} \mathcal{L}(\mathbf{W}_{k}^{\tau-1}, \boldsymbol{\theta}_{k}^0 )\|^2 \\
        & \quad + \frac{1}{2}( {\eta_{k\tau}^{\mathbf{W}}})^2LM
        \end{aligned}
    \end{equation*}
    We assume $\eta_{kt}^{\mathbf{W}}LM_G < 1$ in Assumption \ref{assumption:robbins_monro}, so that
    \begin{equation*}
        -(1-\frac{1}{2}\eta_{kt}^{\mathbf{W}}LM_G) < -\frac{1}{2}
    \end{equation*}
    Based on the descent property for each gradient step, it is straightforward to show the descent property in each block descent step by summing up the gradient steps. By definition, $\mathbf{W}_{k}^{0}$ is the same as $\mathbf{W}_{k-1}^{T}$ and $\boldsymbol{\theta}_{k}^{0}$ is the same as $\boldsymbol{\theta}_{k-1}^{T}$.
    \begin{equation}
        \begin{aligned}
            & \mathbb{E}[\mathcal{L}(\mathbf{W}_{k}^{0}, \boldsymbol{\theta}_{k-1}^{0})] - \mathbb{E}[\mathcal{L}(\mathbf{W}_{k-1}^{0}, \boldsymbol{\theta}_{k-1}^{0} )]\\ 
            \leq & -\frac{1}{2}\sum_{\tau=0}^{T_1}\eta_{kt}^{\mathbf{W}}\mathbb{E}[ \|\nabla_{\mathbf{W}} \mathcal{L}(\mathbf{W}_{k-1}^{\tau}, \boldsymbol{\theta}_{k-1}^0 )\|^2] \\
            &  + \frac{1}{2}LM\sum_{\tau=0}^{T_1} ({\eta_{k\tau}^{\mathbf{W}}})^2
        \end{aligned}
    \end{equation}
    
    \begin{equation}
        \begin{aligned}
            & \mathbb{E}[\mathcal{L}(\mathbf{W}_{k}^{0}, \boldsymbol{\theta}_{k}^{0})] - \mathbb{E}[\mathcal{L}(\mathbf{W}_{k}^{0}, \boldsymbol{\theta}_{k-1}^{0} )] \\
            \leq & -\frac{1}{2}\sum_{\tau=0}^{T_2}\eta_{k\tau}^{\boldsymbol{\theta}} \mathbb{E}[\|\nabla_{\boldsymbol{\theta}} \mathcal{L}(\mathbf{W}_{k}^{0}, \boldsymbol{\theta}_{k-1}^{\tau} )\|^2]
        \end{aligned}
    \end{equation}
    
    Summing the above two inequalities leads to Lemma \ref{lemma:sufficient_descent}.
\end{proof}

\begin{theom} \label{theom:gradient_lim_0}
    Under Assumptions \ref{assumption:gradient_continuous}, \ref{assumption:gradient_variance} and \ref{assumption:robbins_monro}, we have
    \begin{equation}
        \lim_{k \rightarrow \infty} \mathbb{E} [\|\nabla_{\mathbf{W}} \mathcal{L}(\mathbf{W}_{k}^{\tau}, \boldsymbol{\theta}_{k}^0 )\|] =0 \quad \forall \tau \in [0, T_1]
    \end{equation}
    \begin{equation}
        \lim_{k \rightarrow \infty}  \mathbb{E} [ \|\nabla_{\boldsymbol{\theta}} \mathcal{L}(\mathbf{W}_{k}^{0}, \boldsymbol{\theta}_{k}^{\tau} )\|] =0 \quad \forall \tau \in [0, T_2]
    \end{equation}
\end{theom}

\begin{proof}
    The loss function $\mathcal{L}$ is a valid optimization objective function and it must have a lower bound $\mathcal{L}_{\inf}$. Therefore, 
    \begin{equation*}
        \begin{aligned}
             & \mathcal{L}_{\inf} - \mathbb{E}[\mathcal{L}(\mathbf{W}_k^0, \boldsymbol{\theta}_k^0)] \\
             \leq & \mathbb{E}[\mathcal{L}(\mathbf{W}_{K}^{0}, \boldsymbol{\theta}_{K}^{0})] - \mathbb{E}[\mathcal{L}(\mathbf{W}_{0}^{0}, \boldsymbol{\theta}_{0}^{0} )] \\
             \leq & -\frac{1}{2}\sum_{k=0}^K\sum_{\tau=0}^{T_2}\eta_{k\tau}^{\boldsymbol{\theta}} \mathbb{E}[\|\nabla_{\boldsymbol{\theta}} \mathcal{L}(\mathbf{W}_{k}^{0}, \boldsymbol{\theta}_{k}^{\tau} )\|^2] \\
             & -\frac{1}{2}\sum_{k=0}^K\sum_{\tau=0}^{T_1}\eta_{kt}^{\mathbf{W}}\mathbb{E}[ \|\nabla_{\mathbf{W}} \mathcal{L}(\mathbf{W}_{k}^{\tau}, \boldsymbol{\theta}_{k}^0 )\|^2] \\
             & + \frac{1}{2}LM\sum_{k=0}^K\sum_{\tau=0}^{T_1} ({\eta_{k\tau}^{\mathbf{W}}})^2
        \end{aligned}
    \end{equation*}
    The following inequality can be obtained by rearranging the previous inequality.
    \begin{equation*}
        \begin{aligned}
            & \sum_{k=0}^K\sum_{\tau=0}^{T_2}\eta_{k\tau}^{\boldsymbol{\theta}} \mathbb{E}[\|\nabla_{\boldsymbol{\theta}} \mathcal{L}(\mathbf{W}_{k}^{0}, \boldsymbol{\theta}_{k}^{\tau} )\|^2] \\
            & + \sum_{k=0}^K\sum_{\tau=0}^{T_1}\eta_{kt}^{\mathbf{W}}\mathbb{E}[ \|\nabla_{\mathbf{W}} \mathcal{L}(\mathbf{W}_{k}^{\tau}, \boldsymbol{\theta}_{k}^0 )\|^2] \\
            \leq & 2(\mathbb{E}[\mathcal{L}(\mathbf{W}_k^0, \boldsymbol{\theta}_k^0)] - \mathcal{L}_{\inf}) + LM\sum_{k=0}^K\sum_{\tau=0}^{T_1} ({\eta_{k\tau}^{\mathbf{W}}})^2
        \end{aligned}     
    \end{equation*}
    Taking the infimum of the left hand side leads to
    \begin{equation*}
        \begin{aligned}
            & KT\eta^{\boldsymbol{\theta}}\inf\{ \mathbb{E}[\|\nabla_{\boldsymbol{\theta}} \mathcal{L}(\mathbf{W}_{k}^{0}, \boldsymbol{\theta}_{k}^{\tau} )\|^2] \} \\
            & + KT\eta^{\mathbf{W}}\inf\{ \mathbb{E}[ \|\nabla_{\mathbf{W}} \mathcal{L}(\mathbf{W}_{k}^{\tau}, \boldsymbol{\theta}_{k}^0 )\|^2] \} \\
            \leq  & 2(\mathbb{E}[\mathcal{L}(\mathbf{W}_k^0, \boldsymbol{\theta}_k^0)] - \mathcal{L}_{\inf}) + LM\sum_{k=0}^K\sum_{\tau=0}^{T_1} ({\eta_{k\tau}^{\mathbf{W}}})^2
        \end{aligned}
    \end{equation*}
    where $0 < \eta^{\boldsymbol{\theta}} = \min \{\eta_{k\tau}^{\boldsymbol{\theta}} \}$ and $0 < \eta^{\mathbf{W}} = \min \{\eta_{kt}^{\mathbf{W}}\}$. Dividing $K$ on both sides, we have
    \begin{equation*}
        \begin{aligned}
            & T\eta^{\boldsymbol{\theta}}\inf\{ \mathbb{E}[\|\nabla_{\boldsymbol{\theta}} \mathcal{L}(\mathbf{W}_{k}^{0}, \boldsymbol{\theta}_{k}^{\tau} )\|^2] \} \\
            & + T\eta^{\mathbf{W}}\inf\{ \mathbb{E}[ \|\nabla_{\mathbf{W}} \mathcal{L}(\mathbf{W}_{k}^{\tau}, \boldsymbol{\theta}_{k}^0 )\|^2] \} \\
            \leq  & \frac{1}{K}[2(\mathbb{E}[\mathcal{L}(\mathbf{W}_k^0, \boldsymbol{\theta}_k^0)] - \mathcal{L}_{\inf}) + LM\sum_{k=0}^K\sum_{\tau=0}^{T_1} ({\eta_{k\tau}^{\mathbf{W}}})^2]
        \end{aligned}
    \end{equation*}
    Since $2(\mathbb{E}[\mathcal{L}(\mathbf{W}_k^0, \boldsymbol{\theta}_k^0)] - \mathcal{L}_{\inf})$ must be bounded, and $LM\sum_{k=0}^K\sum_{\tau=0}^{T_1} ({\eta_{k\tau}^{\mathbf{W}}})^2$ converges to a finite limit when $K$ increases due to Assumption \ref{assumption:robbins_monro}. Therefore, the right hand side of the inequality converges to 0 as $K \rightarrow \infty$. Note that $T$, $\eta^{\boldsymbol{\theta}}$ and $\eta^{\mathbf{W}}$ are finite scalars. Consequently, the sequences of $\inf \{\mathbb{E}[\|\nabla_{\mathbf{W}} \mathcal{L}(\mathbf{W}_{k}^{\tau}, \boldsymbol{\theta}_{k}^0 )\|^2] \}$ and $\inf \{\mathbb{E}[\|\nabla_{\boldsymbol{\theta}} \mathcal{L}(\mathbf{W}_{k}^{0}, \boldsymbol{\theta}_{k}^{\tau} )\|^2] \}$ must converge to 0.
    \begin{equation*}
        \lim_{k \rightarrow \infty} \inf \mathbb{E} [\|\nabla_{\mathbf{W}} \mathcal{L}(\mathbf{W}_{k}^{\tau}, \boldsymbol{\theta}_{k}^0 )\|^2] =0 \quad \forall \tau \in [0, T_1]
    \end{equation*}
    \begin{equation*}
        \lim_{k \rightarrow \infty} \inf \mathbb{E} [ \|\nabla_{\boldsymbol{\theta}} \mathcal{L}(\mathbf{W}_{k}^{0}, \boldsymbol{\theta}_{k}^{\tau} )\|^2] =0 \quad \forall \tau \in [0, T_2]
    \end{equation*}
    $\mathcal{L}$ is twice differentiable in our case. Therefore, the infimum of the above two equations can be omitted according to Corollary 4.12 in \citeApp{bottou:Curtis:2018}. Consequently, the partial derivatives converge to 0 as $k \rightarrow \infty$.
\end{proof}

\begin{theom}
    Let $\{\mathbf{W}_k^0, \boldsymbol{\theta}_k^0\}$ be a sequence generated by the proposed training algorithm. The accumulation point of this sequence converges to a stationary point where $\mathbb{E} [ \|\nabla\mathcal{L}(\mathbf{W}_k^0, \boldsymbol{\theta}_k^0) \|] = 0$ as $k \rightarrow \infty$.
\end{theom}
\begin{proof}
    Based on triangle inequality, we have
    \begin{equation*}
        \begin{aligned}
            & \mathbb{E} [ \|\nabla\mathcal{L}(\mathbf{W}_k^0, \boldsymbol{\theta}_k^0) \|] \\
            \leq & \mathbb{E} [\|\nabla_{\mathbf{W}} \mathcal{L}(\mathbf{W}_{k}^{0}, \boldsymbol{\theta}_{k}^0 )\|] + \mathbb{E} [ \|\nabla_{\boldsymbol{\theta}} \mathcal{L}(\mathbf{W}_{k}^{0}, \boldsymbol{\theta}_{k}^{0} )\|]
        \end{aligned}       
    \end{equation*}
    The right hand side converges to 0 as $k \rightarrow \infty$ because of Theorem \ref{theom:gradient_lim_0}. Therefore, $\mathbb{E} [ \|\nabla\mathcal{L}(\mathbf{W}_k^0, \boldsymbol{\theta}_k^0) \|]$ converges to 0 as $k \rightarrow \infty$
\end{proof}

\bibliographyApp{../bib/abbreviations,../bib/articles,../bib/proceedings,../bib/books}
\bibliographystyleApp{plain}

\end{document}